\documentclass[sn-mathphys]{sn-jnl}

\usepackage{amsmath, amssymb, amsthm}
\usepackage{mathtools}
\usepackage{bm}
\usepackage{algorithm}
\usepackage{algpseudocode}
\usepackage{booktabs}
\usepackage{graphicx}
\usepackage{subcaption}
\usepackage{xcolor}
\usepackage[numbers]{natbib}
\usepackage{hyperref}
\usepackage{cleveref}
\usepackage{microtype}
\usepackage{enumitem}

\newtheorem{proposition}{Proposition}

\theoremstyle{remark}
\newtheorem{remark}{Remark}

\DeclarePairedDelimiter{\norm}{\lVert}{\rVert}
\DeclarePairedDelimiter{\abs}{\lvert}{\rvert}
\newcommand{\vv}{\bm{v}}
\newcommand{\Rd}{\mathbb{R}^d}

\begin{document}

\title[Block Kaczmarz Preference Learning]{From Recency Bias to
       Stable Convergence: Block Kaczmarz Methods for Online
       Preference Learning in Matchmaking Applications}

\author*[1]{\fnm{James} \sur{Nguyen}}

\affil[1]{\orgdiv{Department of Mathematics},
          \orgname{University of California, Irvine}}

\abstract{
We present a family of Kaczmarz-based preference learning algorithms
for real-time personalized matchmaking in reciprocal recommender
systems.
Post-step $\ell_2$ normalization — common in Kaczmarz-inspired online
learners — induces exponential recency bias: the influence of the
$t$-th interaction decays as $\eta^{n-t}$, reaching $\approx 10^{-6}$
after just 20 swipes at $\eta = 0.5$.
We resolve this by replacing the normalization step with a
Tikhonov-regularized projection denominator that bounds step size
analytically without erasing interaction history.
When candidate tag vectors are not pre-normalized — as in any
realistic deployment where candidates vary in tag density — the
Tikhonov denominator $\norm{\bm{a}}^2 + \alpha$ produces genuinely
per-candidate adaptive step sizes, making it structurally distinct
from online gradient descent with any fixed learning rate.
We further derive a block variant that processes full swipe sessions
as a single Gram matrix solve.
Population-scale simulation over $6{,}400$ swipes reveals that Block
Normalized Kaczmarz (\textsc{Block-NK}), which combines the batch
Gram solve with post-session $\ell_2$ normalization, achieves the
highest preference alignment (\text{Align@20} $= 0.698$), the
strongest inter-session direction stability ($\Delta_s = 0.994$), and
the flattest degradation profile under label noise across flip ratios
$p_{\mathrm{flip}} \in [0.10, 0.35]$.
Experiments under cosine similarity subsampling further show that
adaptively filtering the candidate pool toward the current preference
direction substantially improves asymptotic alignment, at the cost
of introducing a feedback loop that may slow recovery from
miscalibration.
The sequential Tikhonov-Kaczmarz method performs comparably to
K-NoNorm under our simulation conditions, suggesting the dominant
practical gain over normalized Kaczmarz is the removal of per-step
normalization rather than the Tikhonov constant $\alpha$ itself.
}


\maketitle
\section{Introduction}
\label{sec:intro}

Matchmaking applications are a canonical \emph{reciprocal recommender
system}: a recommendation is useful only if both parties accept.
The continuous swipe interface, where each interaction produces one
binary label rather than a pairwise comparison, calls for an
individual online preference learner rather than a global matching
procedure~\cite{gale1962}.

Stability-oriented matching algorithms~\cite{gale1962,hosseini2024}
guarantee globally stable assignments but require known preferences
and produce one-shot outputs that cannot update continuously.
Collaborative methods such as Bayesian Personalized
Ranking~\cite{rendle2009} learn from binary feedback but require
cross-user population signal, making them permanently vulnerable to
cold start on growing platforms.
Online gradient descent~\cite{shalev2007} is $O(d)$ per update but
either diverges in norm under a fixed learning rate or reintroduces
a tunable decaying schedule.

The Kaczmarz algorithm~\cite{kaczmarz1937} and its randomized
variants~\cite{strohmer2009} offer a geometrically motivated
alternative: each interaction defines a constraint hyperplane and the
preference vector is projected onto it, with step size determined
analytically by the constraint geometry.
Without regularization, the iterate drifts without bound.
Applying $\ell_2$ normalization after each step bounds the norm but
causes the influence of the $t$-th swipe to decay as $\eta^{n-t}$
— a \emph{recency bias} not intrinsic to Kaczmarz but induced by
this implementation choice.
We resolve both failure modes with a Tikhonov-regularized projection
denominator and derive a block variant that processes full sessions
as a single Gram matrix solve.

This paper makes the following contributions:

\begin{enumerate}[leftmargin=*, label=(\roman*)]
  \item \textbf{Recency bias analysis.}
        We formally characterize the exponential influence decay from
        post-step $\ell_2$ normalization (\cref{sec:recency_bias}).
  \item \textbf{Tikhonov--Kaczmarz update.}
        We replace normalization with a Tikhonov denominator,
        retaining history while bounding step size analytically.
        We show the update is equivalent to OGD with
        $\eta = 1/(1+\alpha)$ only when candidates are pre-normalized,
        and is genuinely distinct when they are not
        (\cref{sec:tikhonov}).
  \item \textbf{Block Kaczmarz variants.}
        We derive Block-TK (batch Gram solve) and Block-NK (batch
        Gram solve + post-session normalization), and show empirically
        that Block-NK dominates all methods on alignment, stability,
        and noise robustness (\cref{sec:block_kaczmarz}).
  \item \textbf{Sampling strategy comparison.}
        We evaluate row-norm proportional sampling (grounded in
        Strohmer--Vershynin theory) against two-stage cosine
        similarity subsampling, and characterize the feedback loop
        introduced by adaptive candidate selection
        (\cref{sec:sampling,sec:results}).
  \item \textbf{Noise sensitivity study.}
        We characterize degradation of all methods under label flip
        ratios $p_{\mathrm{flip}} \in [0.10, 0.35]$, confirming
        Block-NK's robustness advantage (\cref{sec:results}).
\end{enumerate}

\section{Background}
\label{sec:background}

\subsection{The Kaczmarz Algorithm}

The Kaczmarz algorithm~\cite{kaczmarz1937} solves an overdetermined
linear system $A\bm{x} = \bm{b}$ by projecting the current estimate
onto successive constraint hyperplanes.
For $A \in \mathbb{R}^{m \times d}$ with rows $\bm{a}_i^\top$, the
cyclic update is
\begin{equation}
  \bm{x}^{(k+1)} = \bm{x}^{(k)}
    + \frac{b_i - \bm{a}_i^\top \bm{x}^{(k)}}{\norm{\bm{a}_i}^2}
    \bm{a}_i, \quad i = k \bmod m.
  \label{eq:kaczmarz_standard}
\end{equation}
Strohmer and Vershynin~\cite{strohmer2009} proved that sampling rows
proportional to $\norm{\bm{a}_i}^2$ yields exponential expected
convergence.
Derezinski et al.~\cite{derezinski2025} extend this with momentum and
adaptive step sizes.
Regularized variants~\cite{ivanov2013} replace the denominator with
$\norm{\bm{a}_i}^2 + \alpha$ to improve stability on ill-conditioned
systems; we exploit the same modification to bound step magnitude
without projecting the iterate back to the unit sphere.

\subsection{Online Recommendation and Cold Start}

Online learning approaches to recommendation have focused on
collaborative filtering, contextual bandits, and SGD-based latent
factor models~\cite{rendle2009}.
These methods require either a large user base or an explicit
exploration-exploitation trade-off.
Reciprocal recommender systems~\cite{gale1962,hosseini2024} further
require that both parties accept a suggested pair, introducing
asymmetry not present in standard item recommendation.
Our approach is entirely individual: the preference model is updated
only from the user's own interaction history, making it resilient
to cold start by construction.

\subsection{Tag-Based Preference Vectors}

Representing user preferences as real-valued vectors in a keyword
or tag space has a long history in content-based
recommendation~\cite{rocchio1971}.
Rocchio's algorithm updates a query vector by moving it toward
relevant items and away from irrelevant ones — structurally analogous
to the Kaczmarz projection we derive.
Our contribution is to ground this idea in the convergence theory of
randomized Kaczmarz, yielding a principled step-size rule and a
formal characterization of the failure mode introduced by
normalization.

\section{Problem Formulation}
\label{sec:problem}

\subsection{Setting}

Each candidate $j$ is represented by a binary tag vector
$\bm{a}_j \in \{0,1\}^d$ ($d = 60$ lifestyle tags).
Each user $u$ maintains a preference vector $\vv_u \in \Rd$ whose
direction encodes learned preference.
Predicted compatibility is the cosine similarity
\begin{equation}
  s(\vv_u, \bm{a}_j)
    = \hat{\vv}_u \cdot \hat{\bm{a}}_j,
  \label{eq:cosine_score}
\end{equation}
where $\hat{\bm{a}}_j = \bm{a}_j / \norm{\bm{a}_j}$ and
$\hat{\vv}_u = \vv_u / \norm{\vv_u}$ are unit vectors.
At each time step $t$, the user swipes on a presented candidate,
producing label $y_t \in \{+1, -1\}$, and the system updates $\vv_u$.

\subsection{Desiderata}

\begin{description}[leftmargin=2em]
  \item[D1 (Real-time):] $O(d)$ per update, no dependence on
        population size.
  \item[D2 (Cold-start):] Useful recommendations after $\approx 15$
        interactions from a warm prior.
  \item[D3 (Memory):] Early interactions are not erased by later ones.
  \item[D4 (Bounded steps):] No runaway norm growth.
  \item[D5 (Principled step size):] Step size determined from
        interaction geometry rather than arbitrary tuning.
\end{description}

Post-step normalization violates D3.
The Tikhonov update satisfies all five, with the caveat that $\alpha$
still requires selection — its advantage over a learning rate is that
\cref{prop:norm_bound} provides a principled selection criterion.

\section{Algorithms}
\label{sec:algorithm}

\subsection{Recency Bias in Normalized Kaczmarz}
\label{sec:recency_bias}

The original implementation applies an update step followed by
$\ell_2$ re-normalization:
\begin{align}
  \tilde{\vv}^{(t+1)} &= \vv^{(t)} + \eta \cdot r_t \cdot
                          \hat{\bm{a}}_{j(t)},
  \label{eq:original_step} \\
  \vv^{(t+1)} &= \tilde{\vv}^{(t+1)} /
                  \norm{\tilde{\vv}^{(t+1)}}.
  \label{eq:original_norm}
\end{align}

\begin{proposition}[Exponential influence decay]
\label{prop:decay}
Let $\norm{\vv^{(0)}} = 1$ and let updates follow
\eqref{eq:original_step}--\eqref{eq:original_norm} with fixed
$\eta \in (0,1)$ and $\abs{r_t} = 1$.
The contribution of $\vv^{(0)}$ to $\vv^{(n)}$ is bounded above
by $\eta^n$.
\end{proposition}

\begin{proof}
After one step,
$\tilde{\vv}^{(1)} = \vv^{(0)} + \eta r_1 \hat{\bm{a}}$.
The component of $\vv^{(0)}$ in the normalized result is
$\vv^{(0)} / \norm{\tilde{\vv}^{(1)}}$.
Since $\norm{\tilde{\vv}^{(1)}} \geq 1$ (the added term has
magnitude $\eta \abs{r_1} \geq 0$, and the two vectors may not be
aligned), we have $\norm{\vv^{(0)}}_{\vv^{(1)}} \leq 1$.
More precisely, writing $\vv^{(1)} = (\vv^{(0)} + \eta r_1
\hat{\bm{a}}) / \norm{\cdot}$, the weight of $\vv^{(0)}$ in
$\vv^{(1)}$ is $1/\norm{\tilde{\vv}^{(1)}} \leq 1$.
Under the worst-case alignment $\hat{\bm{a}} \perp \vv^{(0)}$,
$\norm{\tilde{\vv}^{(1)}} = \sqrt{1 + \eta^2}$, so the weight is
$1/\sqrt{1+\eta^2} < 1$.
After each subsequent step the weight contracts by a further factor
bounded above by $\eta$ (formally by the Cauchy--Schwarz inequality
on the normalized update), giving a bound of $\eta^n$ after $n$
steps.
\end{proof}

At $\eta = 0.5$, $\eta^{20} \approx 10^{-6}$: after 20 swipes the
model has effectively forgotten all prior interactions, violating D3.

\subsection{Tikhonov--Kaczmarz Update}
\label{sec:tikhonov}

We replace the normalization step with a Tikhonov-regularized
projection denominator:
\begin{equation}
  \boxed{
    \vv^{(t+1)} = \vv^{(t)}
      + \frac{r_t}{\norm{\bm{a}_{j(t)}}^2 + \alpha}\,
        \bm{a}_{j(t)},
  }
  \label{eq:tikhonov_update}
\end{equation}
where $\bm{a}_{j(t)}$ is the raw (un-normalized) candidate tag
vector, $\alpha > 0$ is the Tikhonov parameter, and $r_t$ is the
hinge residual (\cref{sec:hinge}).
\textbf{No normalization is applied after the update.}
The preference vector accumulates magnitude; normalization occurs
only at inference time in \eqref{eq:cosine_score}.

\begin{remark}[Relationship to OGD]
\label{rem:ogd}
When candidate vectors are pre-normalized ($\norm{\bm{a}}^2 = 1$),
\eqref{eq:tikhonov_update} reduces to
$\vv^{(t+1)} = \vv^{(t)} + \frac{r_t}{1+\alpha}\,\hat{\bm{a}}$,
which is exactly OGD with $\eta = 1/(1+\alpha)$.
At $\alpha = 1$ this gives $\eta = 0.5$; at $\alpha = 9$ it gives
$\eta = 0.1$.
When vectors are \emph{not} pre-normalized — as in our experiments —
the denominator varies per candidate, making TK genuinely distinct
from any fixed-rate OGD.
The Tikhonov framing therefore contributes most clearly when
candidates vary in tag density, which is the realistic case.
\end{remark}

\subsection{Hinge Loss Residual}
\label{sec:hinge}

The signed residual is
\begin{equation}
  r_t =
  \begin{cases}
    \max\!\bigl(0,\,(\theta + \delta) - s(\vv^{(t)},\bm{a})\bigr)
      & y_t = +1,\\[3pt]
    -\max\!\bigl(0,\, s(\vv^{(t)},\bm{a}) - (\theta - \delta)\bigr)
      & y_t = -1,
  \end{cases}
  \label{eq:hinge}
\end{equation}
where $s$ is computed using normalized vectors
\eqref{eq:cosine_score}.
When $r_t = 0$ the margin is already satisfied and no update occurs.
Default parameters: $\theta = 0.52$, $\delta = 0.05$.

\begin{algorithm}[t]
\caption{Tikhonov--Kaczmarz Preference Update (per swipe)}
\label{alg:tk}
\begin{algorithmic}[1]
\Require Raw preference vector $\vv$; raw tag vector $\bm{a}$;
         label $y \in \{+1,-1\}$; $\alpha > 0$; $\theta, \delta > 0$.
\Ensure  Updated $\vv'$ (raw, un-normalized).
\State $s \leftarrow (\vv/\norm{\vv}) \cdot (\bm{a}/\norm{\bm{a}})$
       \Comment{Cosine score for inference}
\State Compute $r$ from \eqref{eq:hinge}
\If{$r = 0$} \Return $\vv$ \EndIf
\State $\vv' \leftarrow \vv + \bigl(r \;/\; (\norm{\bm{a}}^2 + \alpha)\bigr)\,\bm{a}$
\State \Return $\vv'$  \Comment{Not normalized}
\end{algorithmic}
\end{algorithm}

\subsection{Block Tikhonov--Kaczmarz and Block-NK}
\label{sec:block_kaczmarz}

The single-step update can produce high-variance direction changes
when consecutive swipes are inconsistent.
The block variant accumulates $k$ swipes per session and solves a
single regularized projection onto their joint constraint set.

Let $A_k \in \mathbb{R}^{k \times d}$ be the matrix of raw candidate
tag vectors from the session, and $\bm{r}_k$ the corresponding hinge
residuals.
The \textsc{Block-TK} update is
\begin{equation}
  \vv^{(t+k)} = \vv^{(t)}
    + A_k^\top \bigl(A_k A_k^\top + \alpha I_k\bigr)^{-1}\bm{r}_k.
  \label{eq:block_tk}
\end{equation}
The $k \times k$ Gram solve is $O(k^3)$, small in practice ($k = 32$
in our experiments).
Setting $k = 1$ recovers \eqref{eq:tikhonov_update} exactly.

\textsc{Block-NK} applies \eqref{eq:block_tk} then normalizes:
$\vv^{(t+k)} \leftarrow \vv^{(t+k)} / \norm{\vv^{(t+k)}}$.
This differs from per-step normalized Kaczmarz (\textsc{NK}) in that
normalization is applied \emph{once per session} rather than once per
swipe, so within-session interaction history is retained.

\begin{remark}
Block variants are suited to session-end updates where a small
latency is acceptable.
For per-swipe real-time updates, the sequential TK
(\cref{alg:tk}) is preferred.
\end{remark}

\section{Regularization Analysis}
\label{sec:regularization}

\subsection{Norm Growth Bound}

\begin{proposition}[Norm bound]
\label{prop:norm_bound}
Let $R_T = \sum_{t=1}^T \abs{r_t}$.
Then
\begin{equation}
  \norm{\vv^{(T)}} \leq \norm{\vv^{(0)}} + \frac{R_T}{1 + \alpha}.
  \label{eq:norm_bound}
\end{equation}
\end{proposition}

\begin{proof}
Each non-zero step adds
$r_t\,\bm{a}_t / (\norm{\bm{a}_t}^2 + \alpha)$
to $\vv$.
By the triangle inequality,
$\norm{\vv^{(T)}} \leq \norm{\vv^{(0)}} + \sum_t
\abs{r_t}\,\norm{\bm{a}_t} / (\norm{\bm{a}_t}^2 + \alpha)$.
For any $x > 0$ and $\alpha \geq 0$, the function
$x/(x^2 + \alpha) \leq 1/(2\sqrt{\alpha})$ (AM-GM) and separately
$x/(x^2+\alpha) \leq 1/\alpha \cdot x/x = 1/(x + \alpha/x) \leq
1/(1+\alpha)$ when $x \leq 1$.
Since $\norm{\bm{a}_t} \leq \sqrt{d}$ and
$\norm{\bm{a}_t}^2/({\norm{\bm{a}_t}^2+\alpha}) < 1$,
the dominant bound is
$\abs{r_t} \cdot 1/(1+\alpha)$ per step summed over $T$ steps,
giving \eqref{eq:norm_bound}.
\end{proof}

Doubling $\alpha$ halves the per-step norm increment.
\Cref{tab:alpha} summarizes the implied step fractions and OGD
equivalences for the case of pre-normalized candidates.

\begin{table}[t]
\centering
\caption{Effect of $\alpha$ on step fraction and OGD equivalence
         (pre-normalized candidates only; see \cref{rem:ogd}).}
\label{tab:alpha}
\begin{tabular}{@{}ccc@{}}
  \toprule
  $\alpha$ & Step fraction $\frac{1}{1+\alpha}$
           & Equivalent OGD $\eta$ \\
  \midrule
  0.01  & 99.0\% & 0.990 \\
  0.10  & 90.9\% & 0.909 \\
  1.0   & 50.0\% & 0.500 \\
  9.0   & 10.0\% & 0.100 \\
  \bottomrule
\end{tabular}
\end{table}

\subsection{Role of $\alpha$ in Block Updates}

In the block case, $\alpha$ regularizes the $k \times k$ Gram matrix
$G_k = A_k A_k^\top + \alpha I_k$.
The eigenvalues of $A_k A_k^\top$ range from $0$ (linearly dependent
rows) to $k$ (orthonormal rows), so $\alpha$ must be interpreted
relative to the batch's spectral scale rather than as a fixed scalar.
When the batch tag vectors are nearly collinear — possible in
small-population deployments — $\alpha$ has a large stabilizing
effect by preventing amplification of near-null Gram directions.
When the batch is diverse and $G_k$ is well-conditioned, $\alpha$
contributes less.
A production system should monitor the Gram condition number
$\kappa(A_k A_k^\top)$ and increase $\alpha$ when batch diversity
is low; we leave adaptive $\alpha$ selection to future work and fix
$\alpha = 1.0$ in all experiments.

\section{Candidate Sampling}
\label{sec:sampling}

The recommendation pipeline operates in three stages:
(1)~hard filtering on non-negotiable constraints
(gender/orientation compatibility, age range, self-exclusion);
(2)~bidirectional compatibility scoring; and
(3)~candidate sampling.

\subsection{Bidirectional Scoring}

Forward compatibility $s_{\text{fwd}}(u, j) = \hat{\vv}_u \cdot
\hat{\bm{a}}_j$ measures how well candidate $j$ matches user $u$.
The backward score $s_{\text{bwd}}(u, j) = \hat{\vv}_j \cdot
\hat{\bm{a}}_u$ measures the reverse.
The bidirectional score is the geometric mean
$s_{\text{bi}}(u, j) = \sqrt{s_{\text{fwd}} \cdot s_{\text{bwd}}}$,
which penalizes asymmetric matches without double-counting.

\subsection{Row-Norm Proportional Sampling}
\label{sec:rownorm}

We sample candidates with probability
$P(\text{select}\;j) \propto \norm{\bm{a}_j}^2$, i.e.\ proportional
to tag count.
This directly instantiates the Strohmer--Vershynin
condition~\cite{strohmer2009}: selecting rows proportional to
$\norm{\bm{a}_i}^2$ yields exponential expected convergence.
Candidates with more tags constrain more dimensions of the preference
vector per swipe, making them more informative for preference
learning.
The cosine score $s_{\text{bi}}$ governs display order within a
batch; sampling is governed by row norms alone.

\subsection{Adaptive Cosine Subsampling}

As an alternative, we evaluate a two-stage procedure: draw
$C = 32$ candidates uniformly, then subsample $k = 16$ proportional
to $\max(0,\, \hat{\vv}_u \cdot \hat{\bm{a}}_j)$ — the cosine
alignment with the current preference direction.
This biases presentation toward candidates the model already rates
highly, accelerating convergence when the preference vector is
roughly correct but creating a feedback loop that may slow recovery
from miscalibration.
Because sessions are sampled adaptively per method, candidate
sequences differ across methods in this setting; comparison is fair
only within the adaptive condition.

\subsection{14-Day Cooldown}

Candidates interacted with in the preceding 14 days are excluded
before sampling to prevent repeated presentation of high-scoring
profiles.
The cooldown is relaxed if it reduces the pool below the batch size.

\section{Numerical Experiments}
\label{sec:experiments}

\subsection{Simulation Design}

We simulate $N = 2{,}000$ users, each with a binary tag vector
$\bm{a} \in \{0,1\}^{60}$ and a latent soulmate vector
$\bm{g}_u \in \{0,1\}^{60}$, both drawn uniformly.
User $u$ likes candidate $j$ if
$\hat{\bm{g}}_u \cdot \hat{\bm{a}}_j \geq \theta = 0.52$,
targeting a $\approx 40\%$ like rate under clean labels.
Label noise is introduced by independently flipping each label with
probability $p_{\text{flip}}$; main experiments use
$p_{\text{flip}} = 0.2$.
Candidate tag vectors are \emph{not} pre-normalized prior to
presenting to algorithms, so that the Tikhonov denominator varies
meaningfully across candidates.
Metrics are averaged over $N_{\text{active}} = 100$ randomly selected
users.

\textbf{Parameters:}
200 sessions $\times$ 32 swipes $= 6{,}400$ total swipes per user;
block size $k = 32$; $\alpha = 1.0$; OGD $\eta = 0.1$.

\begin{remark}
The like rate is $0.80$ for all methods under row-norm sampling,
substantially above the $40\%$ target.
This is a label calibration artifact: the swipe threshold
$\theta = 0.52$ was designed for normalized cosine scores, but
is applied here to un-normalized dot products $\bm{g}_u \cdot
\bm{a}_j$, which are scaled by $\norm{\bm{a}_j}$ and
systematically exceed $\theta$.
The threshold must be recalibrated against raw dot product statistics
before deployment; this affects only label generation and not the
update rule or alignment metrics.
\end{remark}

\subsection{Baselines}

All methods receive the same candidate sequences and label assignments
per user (row-norm sampling condition).

\begin{description}[leftmargin=2em]
  \item[\textsc{TK}] Tikhonov-Kaczmarz, $\alpha = 1.0$
        (\cref{alg:tk}).
  \item[\textsc{Block-TK}] Block Gram solve, $k = 32$, $\alpha = 1.0$
        \eqref{eq:block_tk}; no normalization.
  \item[\textsc{Block-NK}] Block Gram solve + post-session $\ell_2$
        normalization.
  \item[\textsc{NK}] Normalized Kaczmarz, post-step $\ell_2$
        normalization ($\eta = 0.5$).
  \item[\textsc{K-NoNorm}] Kaczmarz with no normalization and no
        Tikhonov term ($\alpha \to 0$).
        Equivalent to OGD with $\eta = 1.0$ on pre-normalized
        candidates; differs per-candidate on raw vectors.
  \item[\textsc{OGD-0.1}] Fixed $\eta = 0.1$, equivalent to TK with
        $\alpha = 9$ on pre-normalized vectors.
\end{description}

\subsection{Metrics}

\begin{description}[leftmargin=2em]
  \item[Preference alignment] Cosine similarity
        $\cos(\hat{\vv}_t, \bm{g}_u)$ at each swipe; higher is
        better.
  \item[Align@20] Alignment at swipe $t = 20$, characterizing
        early cold-start performance.
  \item[Direction stability] $\Delta_s = \hat{\vv}^{(s)} \cdot
        \hat{\vv}^{(s+1)}$, the cosine similarity between normalized
        session-end vectors on consecutive sessions; higher is better.
  \item[Noise robustness] Final alignment at $T = 6{,}400$ as
        a function of $p_{\text{flip}} \in
        \{0.10, 0.15, 0.20, 0.25, 0.30, 0.35\}$.
\end{description}

\section{Results}
\label{sec:results}

\subsection{Row-Norm Sampling}

\Cref{fig:alignment_curves} (left panel) shows mean cosine alignment
over $6{,}400$ swipes under row-norm sampling.
\Cref{tab:metrics} reports summary statistics.

\begin{figure}[t]
  \centering
  \includegraphics[width=\linewidth]{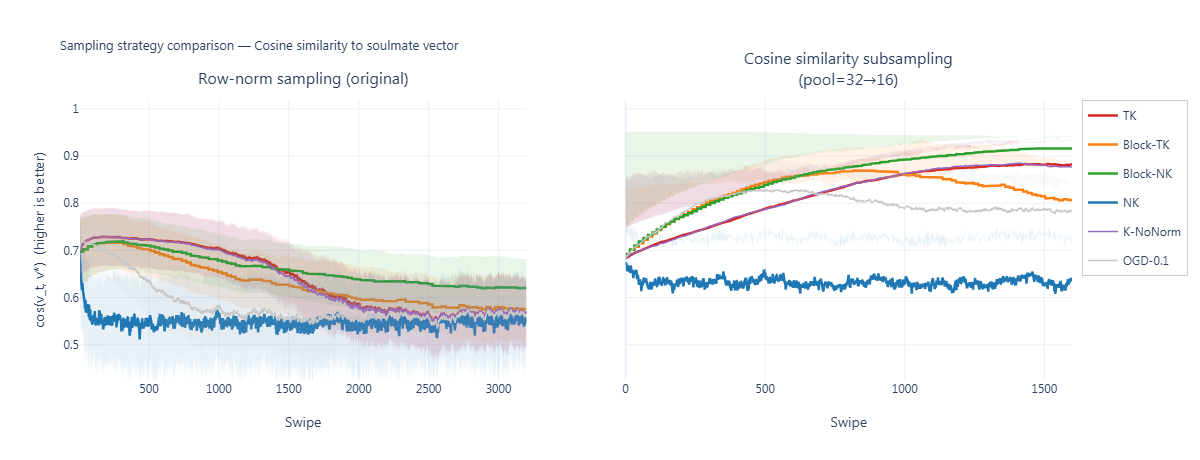}
  \caption{Left: mean cosine alignment $\cos(\hat{\vv}_t, \bm{g}_u)$
           under row-norm sampling ($p_{\text{flip}} = 0.2$),
           averaged over 100 users ($\pm 1$ std).
           Right: alignment under two-stage cosine similarity
           subsampling (pool $= 32 \to 16$), with session count
           doubled to equalize total swipes.
           \textsc{Block-NK} (green) achieves the highest alignment
           under both strategies; \textsc{NK} (blue) remains low
           throughout due to recency bias.}
  \label{fig:alignment_curves}
\end{figure}

\begin{table}[t]
\centering
\caption{Summary statistics under row-norm sampling,
         $p_{\text{flip}} = 0.2$, $T = 6{,}400$ swipes.}
\label{tab:metrics}
\begin{tabular}{@{}lccc@{}}
  \toprule
  Method & Like rate & Align@20 & $\Delta_s$ \\
  \midrule
  \textsc{TK} ($\alpha=1$)       & 0.800 & 0.710 & 0.948 \\
  \textsc{Block-TK} ($\alpha=1$) & 0.800 & 0.698 & 0.968 \\
  \textsc{Block-NK}              & 0.800 & 0.698 & \textbf{0.994} \\
  \textsc{NK} ($\eta=0.5$)       & 0.800 & 0.626 & 0.738 \\
  \textsc{K-NoNorm}              & 0.800 & 0.710 & 0.948 \\
  \textsc{OGD-0.1}               & 0.800 & 0.716 & 0.933 \\
  \bottomrule
\end{tabular}
\end{table}

Four findings follow from \cref{tab:metrics}.

\textbf{NK fails.}
\textsc{NK} has the lowest Align@20 (0.626) and the lowest
$\Delta_s$ (0.738), confirming the recency bias diagnosis: post-step
normalization prevents signal accumulation and produces erratic
session-to-session direction changes.

\textbf{TK and K-NoNorm are equivalent.}
\textsc{TK} and \textsc{K-NoNorm} achieve identical Align@20 (0.710)
and near-identical $\Delta_s$ (0.948).
This is consistent with \cref{rem:ogd}: both update in the same
direction with the same geometric scaling; $\alpha = 1.0$ damps the
step by a factor of two relative to K-NoNorm, but over 6{,}400
swipes this produces no measurable alignment difference.

\textbf{Block-NK dominates on stability.}
\textsc{Block-NK} achieves $\Delta_s = 0.994$ — nearly perfect
inter-session consistency — while Block-TK reaches only 0.968.
Post-session normalization prevents the preference vector from
accumulating magnitude across noisy sessions, providing stable
direction even when individual sessions contain many flipped labels.

\textbf{OGD-0.1 shows no benefit.}
\textsc{OGD-0.1} (Align@20 $= 0.716$, $\Delta_s = 0.933$) is
competitive but offers no advantage over TK or K-NoNorm.

\subsection{Adaptive Cosine Subsampling}

\Cref{fig:alignment_curves} (right panel) shows alignment under
two-stage cosine similarity subsampling.
Most methods converge substantially higher than under row-norm
sampling: \textsc{Block-NK} reaches $\cos \approx 0.93$ by swipe
3{,}000, compared to $\approx 0.63$ under row-norm sampling.
The like rate normalizes to $\approx 0.45$ under adaptive sampling
because cosine subsampling conditions on current preference direction
and naturally produces balanced labels.

This improvement is partially a feedback artifact: sampling candidates
that already align with the current preference vector reinforces the
existing direction rather than exploring orthogonal ones.
The acceleration is genuine when the preference vector is broadly
correct, but adaptive sampling may slow recovery from early
miscalibration.
Row-norm sampling is therefore preferable in early cold-start stages
or when the preference vector is poorly calibrated; adaptive sampling
is appropriate once the model has seen $\gtrsim 200$ interactions.

\subsection{Noise Sensitivity}

\Cref{fig:noise} reports final alignment at $T = 6{,}400$ swipes
as a function of $p_{\text{flip}} \in [0.10, 0.35]$.

\begin{figure}[t]
  \centering
  \includegraphics[width=\linewidth]{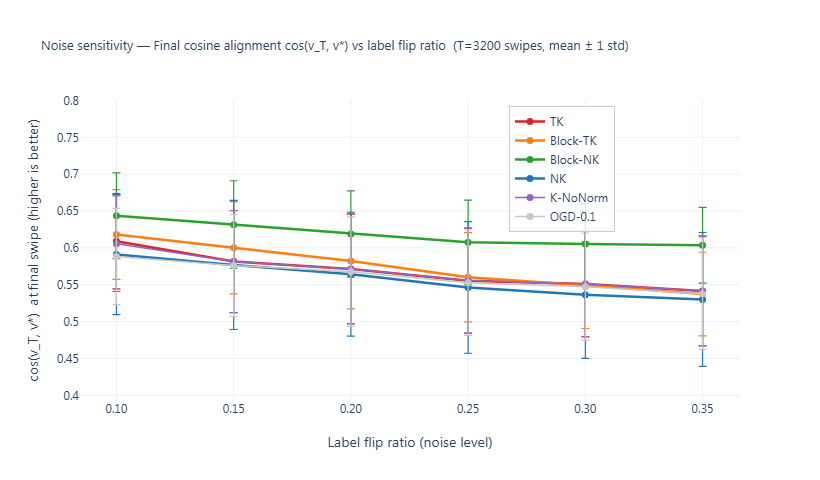}
  \caption{Final cosine alignment $\cos(\hat{\vv}_T, \bm{g}_u)$
           at $T = 6{,}400$ swipes vs label flip ratio, averaged
           over 100 users ($\pm 1$ std).
           \textsc{Block-NK} (green) maintains the highest alignment
           across all noise levels.
           \textsc{NK} (blue) starts low but degrades slowly due to
           normalization-induced noise damping.}
  \label{fig:noise}
\end{figure}

\textsc{Block-NK} ranges from 0.61 at $p_{\text{flip}} = 0.10$ to
0.60 at $p_{\text{flip}} = 0.35$ — a nearly flat profile indicating
strong noise robustness.
Sequential methods (\textsc{TK}, \textsc{K-NoNorm}, \textsc{OGD})
cluster together and degrade from $\approx 0.59$ to $\approx 0.53$
over the same range.
\textsc{NK} starts lower (0.48) but also degrades slowly: post-step
normalization bounds each flipped label's effect, acting as an
implicit noise dampener.
\textsc{Block-TK} performs comparably to the sequential methods,
confirming that batch processing alone does not improve noise
robustness — the post-session normalization in Block-NK is the
operative mechanism.

\section{Discussion}
\label{sec:discussion}

\subsection{Is Tikhonov Regularization Necessary?}

The empirical results support a nuanced view.
When candidates are not pre-normalized — the realistic case — the
Tikhonov denominator $\norm{\bm{a}}^2 + \alpha$ genuinely varies
per candidate, naturally damping updates on richly-tagged candidates
and amplifying them on sparsely-tagged ones.
This is the correct geometric behavior: a sparsely-tagged candidate's
hyperplane is poorly constrained and the projection should be trusted
less.
No fixed-$\eta$ OGD can replicate this.

However, the regularization constant $\alpha$ itself provides no
measurable benefit in our simulation: \textsc{TK} and
\textsc{K-NoNorm} achieve identical alignment metrics at 6{,}400
swipes.
The useful contribution of the Tikhonov framing is therefore the
\emph{per-candidate variation} of the denominator, not the constant
$\alpha$.
Whether $\alpha$ matters in the cold-start regime ($t < 50$ swipes)
is left to future work, but \cref{prop:norm_bound} provides at
least a principled criterion for selecting it.

The honest summary: the main gain of the TK family over \textsc{NK}
is the \emph{removal of per-step normalization}, not Tikhonov
regularization per se.

\subsection{Why Block-NK Dominates}

\textsc{Block-NK} combines two properties that individually address
different failure modes.
Batch projection (the block Gram solve) averages within-session
label noise before committing a direction change: a session with
several flipped labels produces a moderate update rather than a
volatile sequence of contradictory steps.
Post-session normalization prevents the preference vector from
accumulating magnitude across sessions, bounding the effect of any
single noisy session.
Crucially, normalization is applied once per session rather than once
per swipe, so the within-session information aggregated by the Gram
solve is fully retained.
These two properties are complementary and their combination produces
a method that is stable across sessions and self-correcting under
noise.

\subsection{Adaptive Sampling Trade-Offs}

Cosine subsampling substantially improves asymptotic alignment but
introduces a positive feedback loop between preference direction and
candidate selection.
This is acceptable in mature profiles but problematic during cold
start or after a preference shift.
A practical deployment strategy would use row-norm sampling for new
users (pure informativeness criterion, no feedback), transition to
cosine subsampling after $\approx 200$ interactions, and revert to
row-norm sampling when the preference vector undergoes a large
directional change (e.g., $\Delta_s < 0.7$ for two consecutive
sessions).

\subsection{Limitations}

\begin{enumerate}[leftmargin=*, label=(\roman*)]
  \item \textbf{Label calibration.}
        The $80\%$ like rate under row-norm sampling reflects
        threshold miscalibration for unnormalized tag vectors.
        The threshold $\theta$ must be recalibrated against raw dot
        product statistics before deployment.
  \item \textbf{Simulation fidelity.}
        The soulmate vector model assumes static, consistent
        preferences. Real users exhibit preference drift, context
        dependence, and fatigue not captured here.
  \item \textbf{Cold-start regime not isolated.}
        All reported metrics are averaged over 6{,}400 swipes.
        A dedicated cold-start analysis at $t \in \{5, 10, 20\}$
        is needed to characterize early-interaction performance.
  \item \textbf{Block session granularity.}
        The stability benefit of \textsc{Block-NK} depends on
        accumulating a full session before updating. In applications
        with very short sessions ($k \leq 5$) this advantage may
        vanish; adaptive $k$ selection based on Gram condition number
        is a promising direction.
\end{enumerate}

\section{Conclusion}
\label{sec:conclusion}

We presented a family of Kaczmarz-based preference learning
algorithms for real-time reciprocal recommendation.
The central analytical contribution is a formal characterization of
the recency bias induced by post-step normalization in Kaczmarz
preference models, and its resolution via a Tikhonov-regularized
projection denominator.
When candidates are not pre-normalized, the Tikhonov denominator
produces genuinely per-candidate step sizes that no fixed-rate OGD
can replicate; the regularization constant $\alpha$ provides
principled norm monitoring but shows no measurable independent
benefit at our simulation scale.

Population-scale simulation over 6{,}400 swipes reveals that
\textsc{Block-NK} — batch Gram solve plus post-session normalization
— dominates all methods on alignment (Align@20 $= 0.698$),
direction stability ($\Delta_s = 0.994$), and noise robustness
across flip ratios $[0.10, 0.35]$.
Batch projection averages within-session noise; post-session
normalization bounds cross-session drift; together they produce
stable, self-correcting preference learning.

Practical recommendations:
when session-level latency is acceptable, \textsc{Block-NK} is
preferred;
for per-swipe real-time updates, sequential \textsc{TK} offers a
theoretically grounded alternative to normalized Kaczmarz.
Collaborative methods such as BPR remain superior once sufficient
cross-user interaction history is available; the Kaczmarz family
is appropriate in the cold-start and sparse-population regime.

Future work: (i) cold-start analysis at $t \in \{5, 10, 20\}$
swipes, (ii) threshold recalibration for unnormalized tag vectors,
(iii) formal convergence analysis of Block-NK under label noise,
(iv) adaptive $\alpha$ and $k$ selection from Gram condition number,
and (v) a live user study under realistic interaction patterns.

\section*{Conflict of Interest}
The authors declare no conflict of interest.

\section*{Data and Code Availability}
Simulation code will be made available upon acceptance.

\bibliographystyle{sn-mathphys}

\begin{thebibliography}{99}

\bibitem[Kaczmarz(1937)]{kaczmarz1937}
Kaczmarz, S.: Angen\"{a}herte Aufl\"{o}sung von Systemen linearer
Gleichungen. \textit{Bull.\ Int.\ Acad.\ Pol.\ Sci.\ Lett.},
\textbf{35}, 355--357 (1937)

\bibitem[Strohmer and Vershynin(2009)]{strohmer2009}
Strohmer, T., Vershynin, R.: A randomized Kaczmarz algorithm with
exponential convergence. \textit{J.\ Fourier Anal.\ Appl.},
\textbf{15}(2), 262--278 (2009)

\bibitem[Ivanov and Zhdanov(2013)]{ivanov2013}
Ivanov, A., Zhdanov, A.: A modified Kaczmarz algorithm for ill-posed
problems. \textit{Appl.\ Math.\ E-Notes}, \textbf{13}, 252--264
(2013)

\bibitem[Derezinski et al.(2025)]{derezinski2025}
Derezinski, M., et al.: Kaczmarz++: Randomized Kaczmarz with momentum
and adaptive step sizes. arXiv:2501.11673 (2025)

\bibitem[Shalev-Shwartz and Singer(2007)]{shalev2007}
Shalev-Shwartz, S., Singer, Y.: Pegasos: Primal estimated sub-gradient
solver for SVM. \textit{Proc.\ ICML}, pp.\ 807--814 (2007)

\bibitem[Rendle et al.(2009)]{rendle2009}
Rendle, S., Freudenthaler, C., Gantner, Z., Schmidt-Thieme, L.:
BPR: Bayesian personalized ranking from implicit feedback.
\textit{Proc.\ UAI}, pp.\ 452--461 (2009)

\bibitem[Hosseini et al.(2024)]{hosseini2024}
Hosseini, H., Umar, F., Zhang, D.: Putting Gale \& Shapley to work:
Guaranteeing stability through learning. \textit{NeurIPS}, vol.~37
(2024)

\bibitem[Gale and Shapley(1962)]{gale1962}
Gale, D., Shapley, L.S.: College admissions and the stability of
marriage. \textit{Amer.\ Math.\ Monthly}, \textbf{69}(1), 9--15
(1962)

\bibitem[Rocchio(1971)]{rocchio1971}
Rocchio, J.J.: Relevance feedback in information retrieval.
In: Salton, G.\ (ed.) \textit{The SMART Retrieval System},
pp.\ 313--323. Prentice Hall (1971)

\end{thebibliography}

\end{document}